\pdfoutput=1
\documentclass[10pt, twocolumn, twoside]{IEEEtran}
\usepackage{multirow}
\usepackage{makecell}
\usepackage{booktabs}
\usepackage{graphicx}

\usepackage{amsthm}
\usepackage{pdfpages}
\usepackage{textcomp}
\usepackage{epsfig,latexsym}
\usepackage{float}
\usepackage{indentfirst}
\usepackage{amsmath}
\usepackage{amssymb}
\usepackage{xcolor}
\usepackage{CJKutf8}
\newcommand{\bm}[1]{\mbox{\boldmath{$#1$}}}

\usepackage{times}
\usepackage{subfigure}
\usepackage{psfrag}
\usepackage{cite}
\usepackage{flushend}  
\usepackage{epstopdf}
\usepackage{fancyhdr}
\usepackage{stfloats}
\usepackage{color}
\usepackage[noend]{algpseudocode}
\usepackage{hyperref}
\usepackage{algorithmicx,algorithm}
\usepackage{tabularx}
\def\diag{\textrm{diag}}

\newtheorem{lemma}{Lemma}
\newtheorem{remark}{\bf{Remark}}

\begin{document}%
\begin{CJK}{UTF8}{gbsn}
\title{Nash Equilibrium Between Consumer Electronic  Devices and DoS Attacker for Distributed IoT-enabled RSE Systems}
\author{\IEEEauthorblockN{Gengcan Chen, Donghong Cai,~\IEEEmembership{Member,~IEEE}, Zahid Khan,~\IEEEmembership{Senior Member,~IEEE}, Jawad Ahmad,~\IEEEmembership{Senior Member,~IEEE}, and Wadii Boulila,~\IEEEmembership{Senior Member,~IEEE}}
\thanks{G. Chen and D. Cai are with the College of Cyber Security, Jinan University, Guangzhou 510632, China (e-mail: gengcanchen@qq.com; dhcai@jnu.edu.cn).
%
}
\thanks{Z. Khan and Wadii Boulila are with the Robotics and Internet-of-Things Laboratory, Prince Sultan University, Riyadh, Saudi Arabia (e-mail: zskhan@psu.edu.sa;wboulila@psu.edu.sa)}
\thanks{Jawad Ahmad is with the Cybersecurity Center, Prince Mohammad Bin Fahd University, Alkhobar, Saudi Arabia.
(Email: jahmad@pmu.edu.sa). }
%

}
 \vspace{1.5em}

\maketitle
\vspace{-2.0em}

\begin{abstract}
In electronic consumer Internet of Things (IoT), {consumer electronic devices as edge devices require less computational overhead and} the remote state estimation (RSE) of consumer electronic devices is always at risk of denial-of-service (DoS) attacks. Therefore, the adversarial strategy between consumer electronic devices and DoS attackers is critical. This paper focuses on the adversarial strategy between consumer electronic devices and DoS attackers in IoT-enabled RSE Systems. We first propose a remote joint estimation model for distributed measurements to effectively reduce consumer electronic device workload and minimize data leakage risks. The Kalman filter is deployed on the remote estimator, and the DoS attacks with open-loop as well as closed-loop are considered. We further introduce advanced reinforcement learning techniques, including centralized and distributed Minimax-DQN, to address high-dimensional decision-making challenges in both open-loop and closed-loop scenarios. Especially, the Q-network instead of the Q-table is used in the proposed approaches, which effectively solves the challenge of Q-learning. Moreover, the proposed distributed Minimax-DQN reduces the action space to expedite the search for Nash {Equilibrium} (NE). The experimental results validate that the proposed model can expeditiously restore the RSE error covariance to a stable state in the presence of DoS attacks, exhibiting notable attack robustness. The proposed centralized and distributed Minimax-DQN effectively resolves the NE in both open and closed-loop case, showcasing remarkable performance in terms of convergence. It reveals that substantial advantages in both efficiency and stability are achieved compared with the state-of-the-art methods.
\end{abstract}
 \vspace{-0.5em}

 \vspace{-0.5em}
\section{Introduction}
\vspace{0.5em}
With the rapid development of the Internet of Things (IoT), communication networks support massive communication devices or consumer electronic access{\cite{cai2025opensetrffingerprinting}}. The IoT-enabled machine-type communication integrates devices, instruments, vehicles, buildings, and other items embedded with electronics, circuits, software, sensors, and network connectivity\cite{gokhale2018introduction}, which has been widely used in key infrastructure such as public transport networks{\cite{10004210}}, national power grid, and intelligent transportation\cite{mo2011cyber,seiler2020flow}. In the era of Industry 5.0, it has become the focus of attention, especially the wide application of consumer electronic devices, such as wearable sensors and digital cameras, and the building of electronic consumer IoT. However, the electronic consumer IoT involves a lot of personal privacy information vulnerable to malicious network attacks during transmission, which may cause huge loss of life and property\cite{8688434}, and bring great challenges to cyberspace security\cite{nguyen2015survey,khan2018iot}. Moreover, consumer electronic devices have limited battery capacity and computing power. In order to extend battery life and complete complex computing tasks, they need to collaborate with edge servers.

In the electronic consumer IoT, remote state estimation (RSE) is an important task due to limited battery capacity and computing power. However, {advancements in secure communication channels and adaptive access protocols, modern cyber-physical systems remain vulnerable to two predominant cybersecurity threats: spoofing attacks and denial-of-service (DoS) attacks}\cite{cardenas2008secure}. The spoofing attacks can be explained as network attackers transmitting information to edge servers using fake identities. The DoS attacks affect the estimation performance of the system by blocking the wireless channel and causing packet loss\cite{ding2017multi,li2015jamming,qin2017optimal}.
Some anomaly detection methods have been proposed for spoofing attacks in electronic consumer IoT. For example, a stochastic detector with a random threshold for the remote estimator is used to determine whether the received data is correct\cite{7799237}, and $\chi^{2}$ detector is proposed in \cite{7805147}. Besides, some schemes have been proposed to further resist spoofing attacks, including watermark-based KL divergence detector\cite{9142391} and encoding scheme based on pseudo-random numbers\cite{9775725}.

On the other hand, {most of the existing work} has studied the DoS attack in the electronic consumer IoT\cite{6760746,7225119,7172466,ZHANG201252,https://doi.org/10.1002/asjc.1441}. In the initial research of DoS attack, the main achievement is to propose the optimal attack strategy of the attacker\cite{6760746,7225119,7172466} or to propose the detection method of DoS attack\cite{ZHANG201252,https://doi.org/10.1002/asjc.1441}. However, due to both the consumer electronic device and the attacker modifying their strategy to achieve their goal in practice, the traditional methods of thinking only from one side of the problem are obviously not very practical. Therefore, it is an inevitable trend to study the DoS attack in the electronic consumer IoT by considering the strategies of both the attacker and the consumer electronic device. The study of Nash {Equilibrium} (NE) strategy based on offensive and defensive antagonism is mainly carried out under the framework of game theory\cite{li2015jamming,li2016sinr,ding2017multi,ding2018attacks}. Game theory is an appropriate tool for analyzing the dynamic interplay between the consumer electronic device and the attacker due to its ability to capture the adversarial and strategic nature of the situation. {Especially, Wang et al.\cite{WANG2023119134} constructed a Stackelberg game to describe a DoS attack on a normal user launched by several hackers and solved the Stackelberg equilibrium.} In a remote state estimation system, Ding et al. \cite{ding2017multi} proposed a zero-sum game framework for determining the optimal transmission power strategy between the sensor device and the attacker in a wireless communication system based on the respective potential return function.

 Searching for NE in the framework of zero-sum game theory is a challenge. For small-scale games, simple enumeration, line, or arrow methods can be used to solve the problem. Further, Lemke et al. \cite{doi:10.1137/0112033} proposed the Leke-Howson algorithm, which solves the worst-case NE with exponential time complexity. Lipton et al. \cite{10.1145/779928.779933} proposed an approximation algorithm: Lipton-Markakis-Mehta algorithm. In order to solve the distributed NE search problem that may be subjected to denial-of-service attacks in network games, \cite{QIAN2023119080} proposed a fully distributed elastic NE search strategy based on nodes, which can adaptively adjust its control gain according to the local adjacency information of each participant. In addition, \cite{10337768} investigated an NE-seeking algorithm to maintain the resilience of systems subject to DoS attacks and interference in multi-agent systems. \cite{sayin2021decentralized} studied multi-agent reinforcement learning in infinite horizon discounted zero-sum Markov games and developed fully de-coupled Q-learning dynamics. However, these methods have obvious limitations in dealing with complex, changeable, and incomplete information games. Recently, reinforcement learning technology has provided new ideas and methods for resolving the problem of NE. Reinforcement learning, a widely utilized learning method in artificial intelligence\cite{mnih2013playing}, is introduced into the search for the optimal strategy and the equilibrium strategy of the game. In the existing articles, Q-learning is used to find NE based on reinforcement learning under the framework of game theory\cite{li2016sinr,ding2018attacks,ding2020defensive,dai2020distributed}. However, {in the infinite time state game, it is necessary to build a huge Q table. It is a challenge for the preservation of high-dimensional data, which is not a friendly solution for the real complex environments}. The recent advancements in Deep Reinforcement Learning (DRL), specifically Mnih et al.\cite{mnih2015human} proposed the Deep Q-Network (DQN) algorithm offer innovative solutions to similar challenges in the Markov decision process (MDP). Combining deep learning{\cite{10462488,shan2024unsupervised}}, using the neural network instead of Q-table, DQN has a good performance in solving the problem of Q-learning\cite{fan2020theoretical}. Therefore, DQN is commonly used to solve problems with higher dimensions and more states. For example, \cite{9627967} employed DQN instead of Q-learning to optimize the cooperative spectrum sensing system facing primary user emulation attacks. \cite{YANG2022279} developed DQN to solve the MDP problem in a scalable and model-free manner. The existing methods have proven that NE can be found for sensor device and DoS attackers\cite{dolk2016event,feng2020networked}, but they assume that the attacker knows the system parameters and the local sensor device and the DoS attacker is aware of other players' behavior. Moreover, previous research has assumed that the consumer electronic device possesses ample computing power and can transmit local state estimates to remote estimators. These assumptions may not be entirely realistic.

In order to overcome the above shortcomings, this paper introduces a distributed remote joint estimation model and proposes a Minimax-DQN algorithm to address the NE between consumer electronic devices and DoS attacks in IoT-enabled RSE Systems. This paper mainly studies the remote state joint estimation problem of the same target multi-consumer electronic device multi-channel under DoS attack. Centralized reinforcement learning in open-loop cases and distributed reinforcement learning in open-loop and closed-loop cases are discussed. The main contributions of this paper are as follows:
\begin{itemize}
  \item We design a remote joint estimation model for distributed measurements. Compared with the traditional target measurement using a single consumer electronic device, it has higher estimation precision and better estimation effect. In addition, the traditional remote state estimation is to deploy the Kalman filter locally in the consumer electronic device. However, we deploy the Kalman filter at one end of the remote estimator. It can not only reduce the calculation burden of consumer electronic devices but also reduce the data out of the domain. Therefore, it can reduce the risk of information leakage, and it is more suitable for the actual scene.
  \item Then, two methods including centralized and distributed Minimax-DQN algorithm, are proposed for the open-loop case, where the information between consumer electronic devices and the DoS attacker formulated as a Markov process is symmetric. Both methods employ Q-network instead of Q-table, which is more suitable for dealing with the game under a complex environment and continuing state compared with Q-learning. In addition, distributed Minimax-DQN narrows down the action space to expedite the search for NE.
  \item Furthermore, in order to be closer to the actual situation, a closed-loop case is investigated, where information between consumer electronic devices and the DoS attacker is asymmetrical. It is a partially observed MDP (POMDP), we convert it into a game based on belief states to solve the POMDP problem. Then, the distributed Minimax-DQN algorithm is employed to find the NE in a closed-loop case.
\end{itemize}

The subsequent section of this article is presented below. Section \ref{System Model} introduces the mathematical definition of the model and DoS attack and describes the problem. Section \ref{game theory} describes the framework of game theory under this model and the goal to achieve. In section \ref{open loop}, we propose the solution of NE by centralized Minimax-DQN and distributed Minimax-DQN under an open loop. In section \ref{close loop}, we apply distributed Minimax-DQN to find NE in the closed loop. Section \ref{simulation} presents the performance of the proposed method in various settings, along with a comprehensive analysis and experimental outcomes.  Finally, Section \ref{conclusion} summarizes this work.

\textit{Notations}: Let $\mathbb{N}$ and $\mathbb{R}$ denote the sets of nonnegative integers and real numbers, respectively. $\mathbb{R}^{n}$ is the n-dimensional Euclidean space. {$f_{2}\circ f_{1}(x)$} denotes the composition function $f_{2}(f_{1}(x))$. $\diag(e_{1},\dots,e_{n})$ represents the diagonal matrix with its diagonal elements varying from  $e_{1}$ to $e_{n}$. $\mathbb{E}[\cdot]$ represents the expectation of a random variable. $\mathbf{w}\sim \mathcal{N}(\mu, {\sigma^{2}})$ means that $\mathbf{w}$ follows a mean $\mu$ and a variance $\sigma^{2}$ Gaussian distribution. $\dag$ denotes the pseudo-inverse of a matrix.

\section{System Model and DoS Attacks}\label{System Model}
\subsection{Distributed IoT Model}
\vspace{1.0em}
As shown in Fig. \ref{system_model}, we consider the following linear discrete-time state estimation system with multiple consumer electronic devices:
\begin{equation}\label{esf}
\mathbf{x}_{k+1}=\mathbf{A}\mathbf{x}_{k}+\mathbf{w}_{k},
\end{equation}
\begin{equation}\label{asdf}
y_{i,k}=\mathbf{C}_{i}\mathbf{x}_{k}+v_{i,k},
\end{equation}
where $\mathbf{x}_{k} \in\mathbb{R}^{M}$ is the state of the process at the $k$-th time slot, $\mathbf{A}\in\mathbb{R}^{M\times M}$ is a state transition matrix. Each consumer electronic device $i$ measures the state of the same process, and the local measurement output of consumer electronic device $i$ is {denoted} as $y_{i,k}\in\mathbb{R}$. $\mathbf{C}_{i}\in\mathbb{R}^{1\times M}$ is a local measure matrix of the $i$-th consumer electronic device. Moreover, $\mathbf{w}_{k}\sim \mathcal{N}(\mathbf{0}, \mathbf{Q})\in \mathbb{R}^{M}$ and $v_{i,k}\sim\mathcal{N}(0,R_{i})\in\mathbb{R}$ are the corresponding {Additive white Gaussian noises} (AWGNs) for process and measurement. We assume the initial state $\mathbf{x}_{0}\sim\mathcal{N}(\mathbf{0},\bm{\Pi}_{0})$ with $\bm{\Pi}_{0}\in\mathbb{R}^{M\times M}$, which is not influenced by $\mathbf{w}_{k}$ and $v_{i,k}$. Multiple consumer electronic devices are used to measure the same process, and the measurement results are transmitted to the remote estimator through wireless channels, which will effectively improve the accuracy of the system measurement. In the process, it is assume that the tuple$(\mathbf{A},\mathbf{Q}^{1/2})$ can be stabilized and the pair$(\mathbf{A},\mathbf{C})$ satisfies the observability condition, where $\mathbf{C} \triangleq [\mathbf{C}_{1}^{\top}, \ldots ,\mathbf{C}_{n}^{\top}]^{\top}$.

\begin{figure}
\includegraphics[width=1\linewidth]{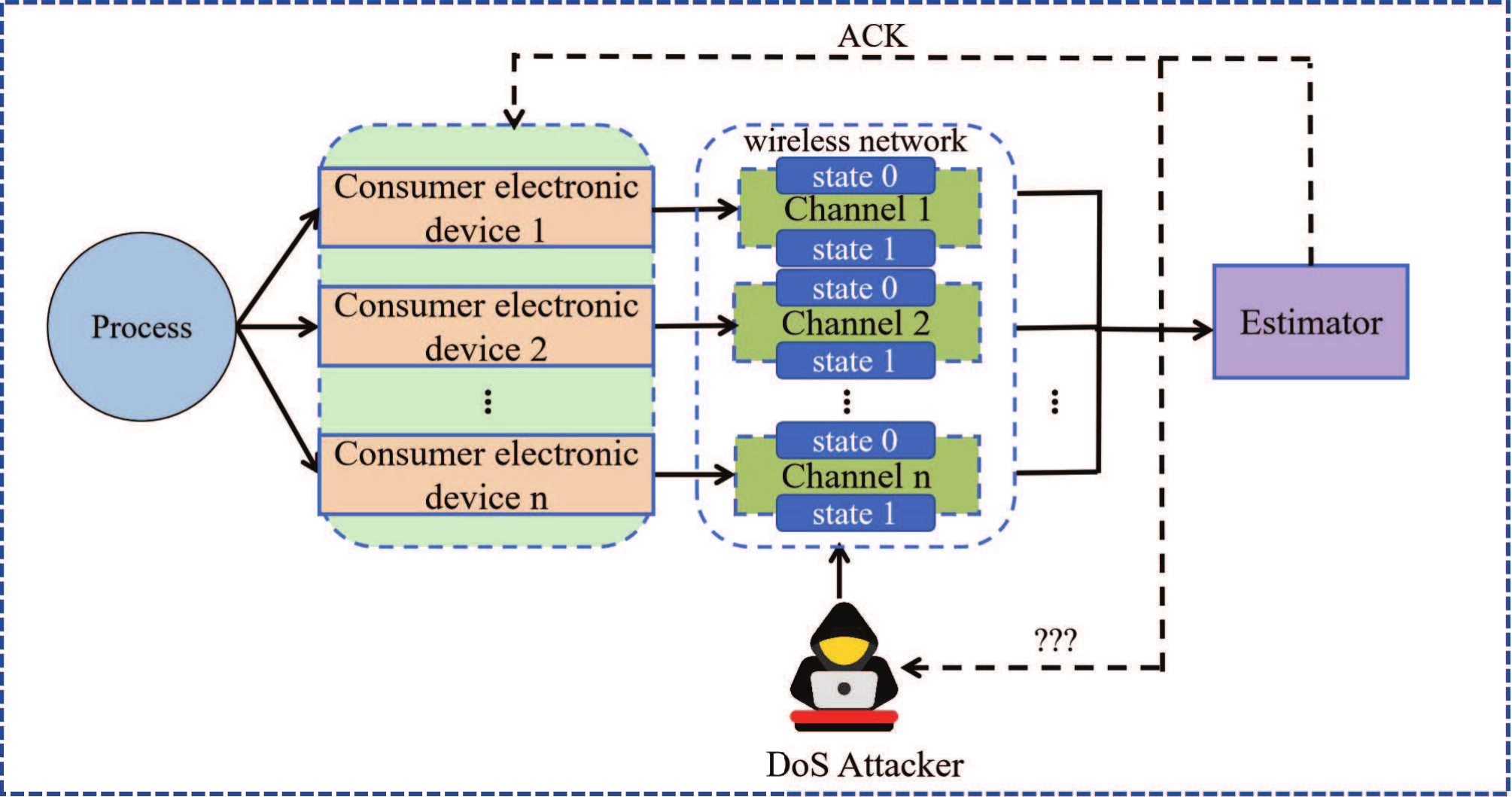}
\caption{Remote State Estimation of Distributed IoT with DoS Attacker.} \label{system_model}
\end{figure}

\subsection{Remote State Estimation}
{Assuming consumer electronic devices are equipped with embedded chips and sufficient battery capacity to support basic computational tasks}. For the $k$-th time slot, the consumer electronic device {$i\in \{1,2,\cdots,n\}$}, makes a raw measurement $y_{i,k}$ and process it to get the innovation:
\begin{equation}
z_{i,k} \triangleq y_{i,k}-\mathbf{C}_{i}\hat{\mathbf{x}}_{k-1}^{-},
\end{equation}
where $\hat{\mathbf{x}}_{k-1}^{-}$ is the feedback from the remote estimator at the previous slot, and $z_{i,k}$ is sent to the remote estimator through channel $i$ in wireless network\cite{7470566}. Especially, the properties of $z_{i,k}$ are given by Lemma 1.
\begin{lemma}\label{lemma1}(\cite{anderson1979optimal}) The innovation $z_{i,k}$ has the following properties:
\begin{enumerate}
  \item $z_{i,k}$ follows zero-mean Gaussian distribution;
  \item $\Sigma_{z_{i}}\triangleq {\mathbb{E}}[z_{i,k}z_{i,k}^{T}]=\mathbf{C}_{i}\bar{\mathbf{P}}_{k}\mathbf{C}_{i}^{T}+R_{i}$;
  \item $z_{i,k}$ is independent of $z_{i,h}$, for $\forall h<k$.
\end{enumerate}
\end{lemma}
\begin{remark}
In traditional local Kalman filters, they employ consumer electronic devices to compute the estimated state and transmit it. In this paper, the consumer electronic devices only compute the innovative $z_{i,k}$ for transmission. Four advantages can be obtained: (i) the consumer electronic device only needs to do simple calculations, effectively reducing the calculation cost of the consumer electronic device. (ii) transmission innovation reduces the consumption of network bandwidth. (iii) innovation is more confidential and reduces the risk of information leakage. (iv) {According to Lemma \ref{lemma1}, the innovation has a known distribution, which makes it easy to detect data tampering}.
\end{remark}

For each communication round, the remote estimator combines multiple received innovations $z_{i,k}, {i\in \{1,2,\cdots,n\}}$, and uses them to estimate the state $\mathbf{x}_{k}$ of the system. In particular, we define the estimation $\hat{\mathbf{x}}_{k} \triangleq \mathbb{E}\{\mathbf{x}_{k}|\mathbf{y}(1:k)\}$ and the corresponding estimation error covariance $\mathbf{P}_{k}\triangleq\mathbb{E}\{(\mathbf{x}_{k}-\hat{\mathbf{x}}_{k})(\mathbf{x}_{k}-\hat{\mathbf{x}}_{k})^T|\mathbf{y}(1:k)\}$. Then the Minimum Mean Squared Error (MMSE) estimations are given as follows:
\begin{subequations}\label{problem-1}
\begin{equation}
\hat{\mathbf{x}}_{k}^{-}=\mathbf{A}\hat{\mathbf{x}}_{k-1},
\end{equation}
\begin{equation}\label{problem-211}
\mathbf{P}_{k}^{-}=\mathbf{A}\mathbf{P}_{k-1}\mathbf{A}^{T}+\mathbf{Q},
\end{equation}
\begin{equation}\label{problem232-1}
\mathbf{K}_{k}=\mathbf{P}_{k}^{-}\mathbf{C}^{T}(\mathbf{C}\mathbf{P}_{k}^{-}\mathbf{C}^{T}+\mathbf{R})^{-1},
\end{equation}
\begin{equation}\label{problem23rt-1}
\hat{\mathbf{x}}_{k}=\hat{\mathbf{x}}_{k}^{-}+\mathbf{K}_{k}(\mathbf{y}_{k}-\mathbf{C}\hat{\mathbf{x}}_{k}^{-}),
\end{equation}
\begin{equation}\label{problemhjy-12}
\mathbf{P}_{k}=(\mathbf{I}-\mathbf{K}_{k}\mathbf{C})\mathbf{P}_{k}^{-},
\end{equation}
\end{subequations}
where $\hat{\mathbf{x}}_{k}^{-}$ and $\hat{\mathbf{x}}_{k}$ are the predicted and updated MMSE estimations of the state $\mathbf{x}_{k}$. $\mathbf{P}_{k}^{-}$ and $\mathbf{P}_{k}$ are covariance matrices of the estimation errors. $\mathbf{K}_{k} \in \mathbb{R}^{M\times n} $ is the Kalman filter gain of the system and $\mathbf{R}=\diag(R_{1},R_{2},...,R_{n})\in \mathbb{R}^{n\times n},\mathbf{y}_{k}=[y_{1,k},y_{2,k},...,y_{n,k}]^{T}$. If the Kalman filter is updated, the remote estimator will provide feedback on the reception of innovations and state estimates $\hat{\mathbf{x}}_{k-1}^{-}$  to the consumer electronic devices. Note that $\mathbf{P}_{k}$ exhibits exponential convergence from any initial condition\cite{anderson1979optimal}. Without loss of generality, we assume that the convergence value is the unique solution $\bar{\mathbf{P}}\succeq \mathbf{0}$ of {$\mathbf{X}=h\circ \tilde{g}(\mathbf{X})$}, where
\begin{equation}\label{equ5}
h(\mathbf{X})\triangleq \mathbf{AXA}^T +\mathbf{Q},
\end{equation}
\begin{equation}
\tilde{g}(\mathbf{X}) \triangleq \mathbf{X}-\mathbf{XC}^T (\mathbf{CXC}^T +\mathbf{R})^{-1}\mathbf{CX}.
\end{equation}
Then, we obtain\cite{li2017detection}
\begin{equation}\label{ewqr1}
\lim_{k \to \infty} \mathbf{P}_{k}^{-}=\bar{\mathbf{P}}.
\end{equation}

For a steady Kalman filter, $\mathbf{P}_{0}=\bar{\mathbf{P}}$.

\subsection{DoS Attack Model}
  In the IoT, consumer electronic devices transmit data through channels to the remote state estimator for joint estimation. While passing through the channels, an attacker may launch a DoS attack to compromise the integrity of the information, thereby degrading system performance.
  We assume that consumer electronic device $i$ sends data using channel $i$, where each channel exists in two different states, i.e., \textbf{status 0} and \textbf{state 1}. For \textbf{state 0}, it doesn't have to pay extra cost, but it has no resistance to DoS attacks. On the other hand, the channel with \textbf{state 1} is effective against a DoS attack but incurs an extra cost $c_{i}$.

  Similarly, it is assumed that the DoS attacker has the capability to attack multiple channels and can decide whether to attack channel $i$ or not, $i \in \{1,\ldots,n\}$. If the DoS attacker attacks the channel $i$, there is an additional fixed cost $c_{i}^{d}$; otherwise, there is no cost.

  Define $\bm{\alpha}_{k}=(\alpha_{k}^{1},\ldots,\alpha_{k}^{n})^{T}$ as the action of consumer electronic devices at $k$-th time slot, where $\alpha_{k}^{i},i \in \{1,\ldots,n\}$ is the action of consumer electronic device $i$. When the consumer electronic device $i$ selects channel $i$ in \textbf{state 1} for data transmission, we {denote} $\alpha_{k}^{i}=1$. When consumer electronic device $i$ selects channel $i$ in \textbf{state 0} for data transmission, we {denote} $\alpha_{k}^{i}=0$ . Similarly, define $\bm{\beta}_{k}=(\beta_{k}^{1},\ldots,\beta_{k}^{n})^{T}$, where $\beta_{k}^{i},i \in \{1,\ldots,n\}$ represents the DoS attacker's decision variable for the channel $i$. $\beta_{k}^{i}=1$ indicate that the DoS attacker launches an attack on channel $i$ at time $k$. Otherwise, $\beta_{k}^{i}=0$. For convenience, it is assumed that only when consumer electronic device $i$ selects channel $i$ in \textbf{state 0} to transmit data, the DoS attacker launches attacks on the channel $i$, the transmitted data will lose packet, i.e., $\alpha_{k}^{i}=0$ and $\beta_{k}^{i}=1$. In other cases, the packet is not lost and always arrives at remote state estimation successfully. Therefore, we {denote} the arrival indicator $\gamma_{i,k}$ of the packet $i$ as follows:
\begin{equation}
\gamma_{i,k}=\left\{
\begin{aligned}
0 & , & (\alpha_{k}^{i},\beta_{k}^{i})=(0,1), \\
1 & , &\text{otherwise}.
\end{aligned}
\right.
\end{equation}

We define $\hat{z}_{i,k}$ as the innovation received by the remote estimator over the network. When packet $i$ can be successfully received by remote estimator, we define $\hat{z}_{i, k} = z_{i, k} $, otherwise we have $\hat{z}_{i, k}=0$. Then, the arrival of innovation $z_{i,k}$ at the remote estimator is defined as
\begin{equation}
\hat{z}_{i,k}=\left\{
\begin{aligned}
z_{i,k} & , & \gamma_{i,k}=1, \\
0 & , &\gamma_{i,k}=0.
\end{aligned}
\right.
\end{equation}

For the sake of description, we define $\bm{\gamma}_{k}=\diag\left(\gamma_{1,k},\ldots,\gamma_{n,k}\right) $.
{If the RSE system is facing DoS attacks, according to Lemma \ref{lemma1},} the Kalman filter updates the system variable $\hat{\mathbf{x}}_{k}$ and the estimation error covariance $\mathbf{P}_{k}$ by
\begin{subequations}\label{pro1}
\begin{equation}
\tilde{\mathbf{K}}_{k}=\mathbf{P}_{k}^{-}\tilde{\mathbf{C}}^{T}(\tilde{\mathbf{C}}\mathbf{P}_{k}^{-}\tilde{\mathbf{C}}^{T}+\tilde{\mathbf{R}})^{\dag},
\end{equation}
\begin{equation}
\hat{\mathbf{x}}_{k}=\hat{\mathbf{x}}_{k}^{-}+{\mathbf{\tilde{K}}}_{k}\hat{\mathbf{z}}_{k},
\end{equation}
\begin{equation}
\mathbf{P}_{k}=(\mathbf{I}-\tilde{\mathbf{K}}_{k}\tilde{\mathbf{C}})\mathbf{P}_{k}^{-},
\end{equation}
\end{subequations}
where {$\tilde{\mathbf{C}}=\bm{\gamma}_{k}\mathbf{C}$, and $\tilde{\mathbf{R}}=\bm{\gamma}_{k}\mathbf{R}\bm{\gamma}_{k}^{T}$. The received innovation $\hat{\mathbf{z}}_{k}$ is defined as $\left[\hat{z}_{1,k},\hat{z}_{2,k},\ldots,\hat{z}_{n,k}\right]^T$.} Define $F(\mathbf{X},\bm{\gamma})\triangleq(\mathbf{I}-\tilde{\mathbf{K}}_{k}\bm{\gamma}_{k}\mathbf{C})h(\mathbf{X})\triangleq(\mathbf{I}-\tilde{\mathbf{K}}_{k}\tilde{\mathbf{C}})h(\mathbf{X})$.
Correspondingly, the covariance $\mathbf{P}_{k}$ of estimation error is defined as
\begin{equation}\label{qrer}
\mathbf{P}_{k}=F(\mathbf{P}_{k-1},\bm{\gamma}_{k}).
\end{equation}
\begin{lemma}\label{mylemma}
In a steady Kalman filter, the  covariance $\mathbf{P}_{k}$ of estimation error in \eqref{qrer} is given by
\begin{align}
\bar{\mathbf{P}}=F(\bar{\mathbf{P}},\mathbf{I}).
\end{align}

\end{lemma}
\begin{proof}
Note that the Kalman filters have attained a steady state, i.e. {\eqref{ewqr1} is valid}.

Furthermore, $\bar{\mathbf{P}} \succeq \mathbf{0}$ is a unique solution that is positive semi-definite of $\mathbf{X}=h\circ \tilde{g}(\mathbf{X})$, we have
\begin{align}
&h(\tilde{g}(\bar{\mathbf{P}}))=\bar{\mathbf{P}}\Leftrightarrow h(\tilde{g}(\lim_{k \to \infty} \mathbf{P}_{k}^{-}))=\bar{\mathbf{P}}.
\end{align}

According to the continuity of the function $\tilde{g}({\cdot})$, it is natural to obtain that:
\begin{equation}\label{2ewr}
 h(\lim_{k \to \infty}\tilde{g}(\mathbf{P}_{k}^{-}))=h(\tilde{g}(\lim_{k \to \infty} \mathbf{P}_{k}^{-}))=\bar{\mathbf{P}}
\end{equation}

According to {\eqref{problem-211}, \eqref{equ5} and \eqref{ewqr1}}, we get
\begin{equation}\label{cgfae}
\lim_{k \to \infty}h(\mathbf{P}_{k-1})=\bar{\mathbf{P}}.
\end{equation}

Because of the continuity of $h(\cdot)$, the following result is obtained:
\begin{equation}\label{1ewr}
h(\lim_{k \to \infty}\mathbf{P}_{k-1})=\bar{\mathbf{P}}.
\end{equation}

{Combining \eqref{ewqr1}}, \eqref{2ewr} and \eqref{1ewr}, it is easy to get
\begin{equation}
\tilde{g}(\bar{\mathbf{P}})=\tilde{g}(\lim_{k \to \infty} \mathbf{P}_{k}^{-})=\lim_{k \to \infty}\mathbf{P}_{k-1}=\bar{\mathbf{P}},
\end{equation}

According to \eqref{cgfae}, we obtain
\begin{equation}
\lim_{k \to \infty} \tilde{g}(h(\mathbf{P}_{k-1}))=\tilde{g}(\lim_{k \to \infty}h(\mathbf{P}_{k-1}))=\tilde{g}(\bar{\mathbf{P}})=\bar{\mathbf{P}}.
\end{equation}
When the system is steady, $\bm{\gamma}_{k}=\mathbf{I}$, then we have $F(\cdot)=\tilde{g}(h(\cdot))$.
\end{proof}

Therefore, when the system is stable and not under DoS attack, the state transition function can ensure that the error covariance is still $\bar{\mathbf{P}}$. Define $\lambda_{k} \triangleq k-l$, where $l$ indicates the time for the first packet loss from the stable state. The relationship between $\lambda_{k}$ and $\mathbf{P}_{k}$ at time $k$ is $\mathbf{P}_{k}=F^{\lambda_{k}}(\bar{\mathbf{P}},\bm{\gamma}_{k})$. Because of the fast convergence of the estimation error covariance, the error covariance will converge to $\bar{\mathbf{P}}$ when there is no packet loss within a certain time.

\section{NE Framework for Distributed IoT Security}\label{game theory}
In this section, the consumer electronic devices and the attacker's decision-making processes in distributed IoT are modeled as an infinite time-horizon game. Then the game will be analyzed based on the NE.
\subsection{Game in Infinite Time Horizon} \label{joij}
    In the game, the players consist of consumer electronic devices and a DoS attacker. We consider that the consumer electronic device computation is continuously generated and the DoS attacker's attack is sustainable over an infinite time range\cite{9174773}.

Game theory is based on two basic assumptions. First, the behavior of both sides of the game is rational, i.e., the decisions of both sides are based on their own interests and the goal is to maximize the interests. Second, each side knows the other's rationality and thus makes the optimal choice among all possible actions. Furthermore, game theory is based on the common knowledge, i.e., each player knows what is agreed in an infinite recursive sense. Therefore, the game can be expressed in terms of six tuples $\mathcal{G}=<L,S,A,\delta,R,\rho>$, i.e.,

\textbf{\textit{Player}}: $L=\{1,\ldots,n,n+1\}$ is players space, where $i\in\{1,\ldots,n\}$ stands for the consumer electronic device $i$ and $i=n+1$ represents the DoS attacker.

{\textbf{\textit{State}}}: $S=\{s_{1},\ldots,s_{k},\ldots\} $ is the set of state space, where $s_{k} \in S$ is the state in the game at time $k$. Let $\mathbf{P}_{k}$ be the state of the game.

\textbf{\textit{Action}}: $A=A_{1} \times,\ldots,\times A_{n} \times A_{n+1}$, where $A_{i},i\in \{1,\ldots,n\}$ is the action space for consumer electronic device $i$ and $A_{n+1}$ is the action space for the DoS attacker. {Consumer} electronic device $i$ selects channel $i$ in \textbf{state 1} or \textbf{state 0}  to transmit its innovation $z_{i,k}$. Similarly, the DoS attacker has a choice to either launch attacks on specific channels or giving up attacking. $\alpha_{k}^{i} \in A_{i}$ is the action of consumer electronic device $i$ at time slot $k$ and the action taken by attacker on the channel $i$ at time $k$ is denoted as $\beta_{k}^{i} \in A_{n+1}$.

\textbf{\textit{State transition $\delta$}}: $S\times A\rightarrow S$, the current state $s_{k}$ and the actions of the consumer electronic devices and the DoS attacker determine the next state $s_{k+1}$ i.e., $s_{k+1}=\delta(s_{k},\bm{\alpha}_{k},\bm{\beta}_{k})$. According to \eqref{qrer}, $\delta=F(\cdot)$.

\textbf{\textit{Reward}}: $R=\{r_{1},\ldots,r_{k},\ldots\}$, where $r_{k}: S \times A \rightarrow \mathbb{R}$ is the common reward function set for consumer electronic devices and the DoS attacker. The payoff at time $k$ for the players can be defined as
\begin{equation}\label{ad1f}
r(\mathbf{P}_{k},\bm{\alpha}_{k},\bm{\beta}_{k})=Tr(\mathbf{P}_{k})+\sum_{i=1}^{n}(c_{i}\alpha_{k}^{i}-c_{\beta}^{i}\beta_{k}^{i}),
\end{equation}
where $Tr(\mathbf{P}_{k})$ is the trace of $\mathbf{P}_{k}$. {The consumer electronic devices aim to minimize the estimation error covariance at the lowest cost, while the attacker conversely intends to maximize the estimation error covariance at the lowest cost.}

\textbf{\textit{Discount factor}}: $\rho \in (0,1)$ is a discounted factor, which  effectively speeds up convergence in the decision-making process by placing more emphasis on immediate payoffs rather than future payoffs.

For the consumer electronic devices and the DoS attacker, given a joint policy $(\pi^{1},\pi^{2})$, the value of state $s_{k}$ is calculated as the discounted cumulative rewards i.e.,
\begin{equation}
v(s_{k},\pi^{1},\pi^{2})=\sum_{k=0}^{\infty}\rho^{k}r(s_{k},\pi^{1}(s_{k}),\pi^{2}(s_{k})),
\end{equation}
where $s_{0}=s$ and $s_{k+1}=\delta(s_{k},\pi^{1}(s_{k}),\pi^{2}(s_{k}))$. Given joint policy $(\pi^{1},\pi^{2})$, the Q-value of the state-action pair $(s,\bm{\alpha},\bm{\beta})$ is denoted by
\begin{equation}
Q(s,\bm{\alpha},\bm{\beta},\pi^{1},\pi^{2})=r(s,\bm{\alpha},\bm{\beta})+v(s,\pi^{1},\pi^{2}),
\end{equation}
where $s_{1}=\delta(s,\bm{\alpha},\bm{\beta})$ and $s_{k+1}=\delta(s_{k},\pi^{1}(s_{k}),\pi^{2}(s_{k}))$.

It is important to point out whether consumer electronic devices and the DoS attacker know each other's behavior will result in whether the game information is symmetric. Then the game for consumer electronic devices and the DoS attacker can be divided into two cases: open-loop and close-loop.
\begin{enumerate}
 \item Open-loop case: the consumer electronic devices and the DoS attacker do not know the behaviors of each other. In particular, the DoS attacker is able to gather the feedback from remote estimator to consumer electronic devices. It forms a complete information static game. As a result, both sides of the game can possess the error covariance matrix $\mathbf{P}_{k}$ of the remote estimator, as well as the corresponding payoff value $r(\mathbf{P}_{k},{\bm{\alpha}_{k},\bm{\beta}_{k}})$ given $\mathbf{P}_{k},{\bm{\alpha}_{k},\bm{\beta}_{k}}$ at time $k$\cite{9174773}.
 \item Close-loop case: the consumer electronic devices and the DoS attacker observe each other's behavior, but the attacker lacks knowledge of the feedback from the remote estimator to the local consumer electronic devices, creating an information asymmetry game.
\end{enumerate}

\subsection{Nash Equilibrium}
The NE is usually the solution to a game problem\cite{maskin1983theory}. In non-cooperative game with two or more players, each player will adjust their action to find a favorable strategy. When NE is reached, an individual can receive no incremental benefit from changing actions, assuming that other players remain constant in their strategies.

In the infinite time-horizon game, both sides of the game strive to search for NE. It is assumed that the NE $(\pi_{*}^{1},\pi_{*}^{2})$ between consumer electronic devices and the DoS attacker can be found, where $\pi_{*}^{1}=(\pi_{1}^{*},\ldots,\pi_{n}^{*}),i\in\{1,\ldots,n\}$. {$\pi_{i}^{*}$ and $\pi_{*}^{2}$ are the strategies of consumer electronic device $i$ and DoS attacker when NE is reached, respectively.} Especially, the NE satisfies the Bellman equation, defined as
\begin{align}
&Q(s,\bm{\alpha},\bm{\beta},\pi_{*}^{1},\pi_{*}^{2})=r(s,\bm{\alpha},\bm{\beta}) \notag \\
&+\rho \max_{\bm{\beta'} \in A_{n+1}}\min_{\bm{\alpha'} \in A_{1}\times,\ldots,\times A_{n}}Q(s',\bm{\alpha'},\bm{\beta'},\pi_{*}^{1},\pi_{*}^{2}),
\end{align}
where $s'=\delta(s,\bm{\alpha},\bm{\beta})$. For all $s \in S$, we have
\begin{equation}
v(s,\pi^{1},\pi_{*}^{2})\leq v(s,\pi_{*}^{1},\pi_{*}^{2})\leq v(s,\pi_{*}^{1},\pi^{2}),
\end{equation}
\begin{equation}
Q(s,\bm{\alpha},\bm{\beta},\pi^{1},\pi_{*}^{2})\leq Q(s,\bm{\alpha},\bm{\beta},\pi_{*}^{1},\pi_{*}^{2})\leq Q(s,\bm{\alpha},\bm{\beta},\pi_{*}^{1},\pi^{2}).
\end{equation}

For both open-loop and close-loop cases, the goal of consumer electronic devices and the attacker is to find a NE.

\section{Open-Loop Security Strategies}\label{open loop}
In this section, we first discuss in the open loop structure of information symmetry. The reinforcement learning algorithm is used to solve the problem in infinite horizon games. In particular, considering the limitations of the existing centralized Q-learning, this paper first proposes the centralized Minimax-DQN approach to find {NE} of the game. Furthermore, in order to reduce the action space and accelerate to find the NE, a distributed Minimax-DQN is proposed.

Under the framework of reinforcement learning, we build an MDP to simulate the interaction between consumer electronic devices and DoS attacker. The components of the reinforcement learning problem are similar to the description of the infinite time range game {in \ref{joij}}. By employing minibatch training, experience replay, and target networks, we can train the network parameters using game data to discover NE.

\subsection{Centralized Minimax-DQN for Infinite Time-Horizon Game}
In recent research, a method called Minimax-DQN has extended Littman's algorithm by incorporating function approximation, similar to DQN\cite{fan2020theoretical}.  To design the Q-network, the state is taken as the input and the Q-value of each action pair $(\bm{\alpha},\bm{\beta})$ is taken as the output. We define the evaluation network as $Q(S,\bm{\alpha},\bm{\beta}|\theta_{k})$, where $\theta_{k}$ is the adjustable parameter. The target network is defined as $Q(S,\bm{\alpha},\bm{\beta}|\theta_{k}^{-})$ and $\theta_{k}^{-}$ is a parameter of it. Given a state-action pair, the consumer electronic devices and the attacker share the same Q function $Q(s,\bm{\alpha},\bm{\beta}|\theta_{k})$ at time $k$ to estimate how well they learn. Define the value function under a NE $(\pi_{*}^{1},\pi_{*}^{2})$ as $Q_{*}(s,\bm{\alpha},\bm{\beta}|\theta^{-})$ satisfying the Bellman equation:
\begin{equation}
Q_{*}(s,\bm{\alpha},\bm{\beta}|\theta^{-})=r(s,\bm{\alpha},\bm{\beta})+\rho Q_{*}(s',\pi_{*}^{1},\pi_{*}^{2}|\theta^{-}),
\end{equation}
where $s'$ is the state generated by taking action $(\bm{\alpha},\bm{\beta})$ in state $s$.  The parameters $\theta_{k}^{-}$ of  DQN can be iteratively adjusted during training to minimize the mean-squared error in the Bellman equation\cite{hou2022deep}. Action $(\bm{\alpha},\bm{\beta})$ obtained by the minimax operation is determined as follow:
\begin{align}
\tilde{y}=r_{k}+\rho\max_{\bm{\beta}'}\min_{\bm{\alpha}'}Q(s_{k+1},\bm{\alpha}',\bm{\beta}'|\theta_{k}^{-}),\label{23adfw}\\
\mathcal{L}_{k}(\theta_{k})=E_{s,\bm{\alpha},\bm{\beta},r,s'}[(\tilde{y}-Q(s,\bm{\alpha},\bm{\beta}|\theta_{k}))^{2}],\label{anmk}
\end{align}
where $r_{k}=r(\mathbf{P}_{k},\bm{\alpha}_{k},\bm{\beta}_{k})$, $\tilde{y}$ is the target or TD-target. It represents the objective that we aim to achieve through updates to the parameters $\theta_{k}$. The weights of the target network are periodically updated every $c$ steps, by transferring the weights from the evaluate network, i.e., $\theta^{-}=\theta$\cite{leong2020deep}. $\mathcal{L}_{k}(\theta_{k})$ is the loss function. Then, we employ the Stochastic Gradient Descent (SGD) optimization algorithm to achieve the best network parameters. This allows us to optimize the loss function and acquire the weight parameters of the DQN, which can be expressed as
\begin{equation}\label{tdxj-1}
\theta_{k+1}=\theta_{k}-\eta\nabla_{\theta_{k}}\mathcal{L}_{k}(\theta_{k}),
\end{equation}
where $\eta$ {denotes} the learning rate and $\nabla_{\theta_{k}}\mathcal{L}_{k}(\theta_{k})$ is the gradient of the loss function with respect to the weights, i.e.,
\begin{equation}\label{tdxj-2}
\nabla_{\theta_{k}}\mathcal{L}_{k}(\theta_{k})=E_{s,\bm{\alpha},\bm{\beta},r,s'}[\tilde{y}-Q(s,\bm{\alpha},\bm{\beta}|\theta_{k})].
\end{equation}

\begin{figure}
\flushleft
\includegraphics[width=1\linewidth]{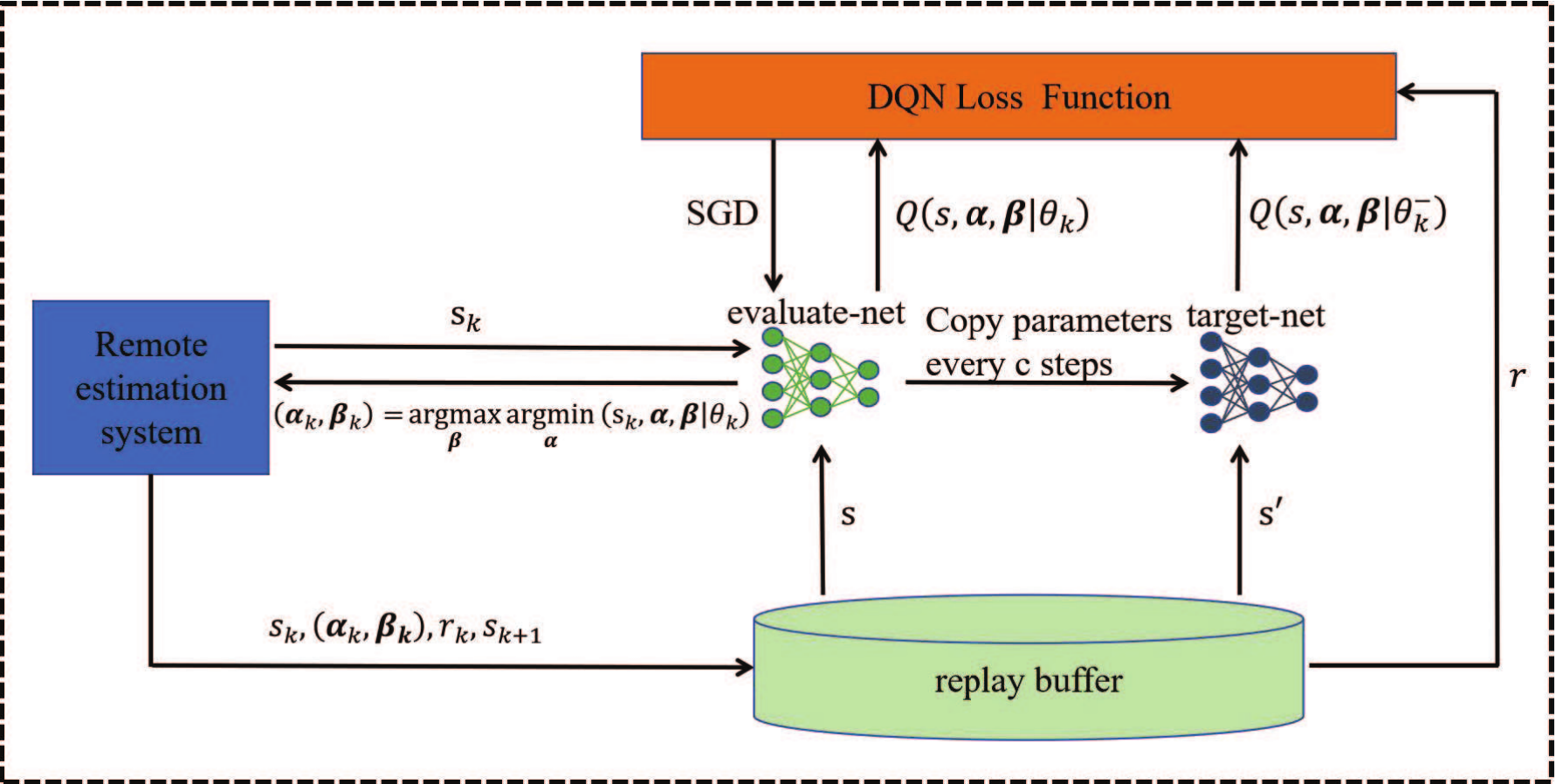}
\caption{Training process of centralized Minimax-DQN.} \label{training process}
\end{figure}
\begin{remark}
 In this scenario, both the consumer electronic devices and the DoS attacker make independent choices for their actions at each time step. As a result, there are $2^n$ strategies available to the consumer electronic devices for the scheduler, and the DoS attacker also has $2^n$ strategies. Consequently, the action space $A$ comprises $2^{2n}$ possible elements. Any $s\in S$ requires $2^{2n}$ space to learn the Q-function.
\end{remark}
\begin{remark}
The Minimax-DQN proposed in this paper, in addiction to using neural networks to approximate functions, empirical replay and target networks are also used to assist. Empirical replay can break the relationship between training data, so that the sample data can meet the independent hypothesis. Meanwhile, the sample data can be used many times, improving the utilization rate of data. The target network uses the same structure as Q-network to enhance the stability of neural network training. The training process of Centralized Minimax-DQN is shown in the Fig. \ref{training process} and the Centralized Minimax-DQN for search NE is shown in Algorithm~\ref{algorithm1}.
\end{remark}

\begin{algorithm}[tp]
\caption{{Centralized Minimax-DQN for NE}}
\label{algorithm1}
 \begin{algorithmic}[1]
 \State \textbf{Input:} Markov game $\mathcal{G}=<L,S,A,\delta,R,\rho>$, replay memory $\mathcal{D}$, minibatch size $N$, exploration probability $\epsilon \in (0,1)$.
 \For {episode = $1$ to $T$}
 \State Initialize replay memory $\mathcal{D}$ to capacity N, Q-network with parameters $\theta$, target-network with parameters $\theta^{-}=\theta$, observation $s_{1}=\bar{\mathbf{P}}$ and $k=1$.
 \While{$\|Q(s_{k},\bm{\alpha},\bm{\beta}|\theta_{k+1})-Q(s_{k},\bm{\alpha},\bm{\beta}|\theta_{k})\|>0$}
 \State With probability $\epsilon$ consumer electronic devices and the DoS attacker select a random action $(\bm{\alpha}_{k},\bm{\beta_{k}})$, otherwise $(\bm{\alpha}_{k},\bm{\beta}_{k}) =\mathop{\arg\max}\limits_{\bm{\beta}}\mathop{\arg\min}\limits_{\bm{\alpha}}Q(s_{k},\bm{\alpha},\bm{\beta}|\theta_{k})$.
 \State Execute actions $(\bm{\alpha}_{k},\bm{\beta}_{k})$, observe next state $s_{k+1}$ and reward $r_{k}$ according to \eqref{qrer}\eqref{ad1f}.
 \State Store transition $(s_{k},\bm{\alpha}_{k},\bm{\beta}_{k},s_{k+1})$ in $\mathcal{D}$.
 \State Sample random minibatch of transitions from $\mathcal{D}$.
 \State Compute $\tilde{y}$ according to \eqref{23adfw}.
 \State Form the loss according to \eqref{anmk}.
 \State Update $\theta$ using \eqref{tdxj-1}.
 \State Every $c$ steps update the target network $\theta_{k}^{-}=\theta_{k}$.
 \State $k\gets k+1$
 \EndWhile
 \State \textbf{EndWhile}
 \EndFor
 \State \textbf{EndFor}
 \end{algorithmic}
\end{algorithm}

\subsection{Distributed Minimax-DQN for Infinite Time-Horizon Game}
Unlike the centralized Minimax-DQN, {consumer electronic devices and attackers} no longer share the same network in distributed reinforcement learning, but train their own networks separately.

For the DoS attacker, we denote the Q-value of state $s\in S$ is $Q^{n+1}(s,\bm{\beta}|\theta_{k}^{a})$, where $\theta_{k}^{a}$ is the parameter for training the DoS attacker's evaluate network. $Q^{n+1}(s,\bm{\beta}|\theta_{k}^{a-})$ is the Q-value of the DoS attackers' target Q-network and $\theta_{k}^{a-}$ is the target-network's parameter. Updates to Q-network is described as follow:
\begin{equation}\label{qgdda}
\begin{aligned}
\tilde{y}^{a}=\left\{
\begin{aligned}
&Q^{n+1}(s,\bm{\beta},\theta_{k}^{a-}), \text{if} (s,\bm{\beta})\neq(s_{k},\bm{\beta}_{k}),\\
&r_{k}+\rho \max_{\bm{\beta}}Q^{n+1}(s_{k+1},\bm{\beta}|\theta_{k}^{a-}), \text{otherwise}.
\end{aligned}
\right.
\end{aligned}
\end{equation}
\begin{equation}\label{f2a1}
\mathcal{L}_{k}^{a}(\theta_{k}^{a})=E_{s,\bm{\beta},r,s'}[(\tilde{y}^{a}-Q^{n+1}(s,\bm{\beta}|\theta_{k}^{a}))^{2}],
\end{equation}
where $s_{k+1}$ is the estimation error covariance at the next time step. $\tilde{y}^{a}$ is the target or TD-target of the DoS attacker. $\mathcal{L}_{k}^{a}(\theta_{k}^{a})$ is the loss function of the DoS attacker's Q-network.

Similar to \eqref{tdxj-1}, \eqref{tdxj-2}, the update of $\theta_{k}^{a}$ can be expressed as
\begin{equation}\label{12aff}
\theta_{k+1}^{a}=\theta_{k}^{a}-\eta^{a}\nabla_{\theta_{k}^{a}}\mathcal{L}_{k}(\theta_{k}^{a}),
\end{equation}
where $\eta^{a}$ {denotes} the learning rate and $\nabla_{\theta_{k}^{a}}\mathcal{L}_{k}(\theta_{k}^{a})$ is the weight gradient of the loss function:
\begin{equation}
\nabla_{\theta_{k}^{a}}\mathcal{L}_{k}(\theta_{k}^{a})=E_{s,\bm{\beta},r,s'}[\tilde{y}^{a}-Q(s,\bm{\beta}|\theta_{k}^{a})].
\end{equation}

The policy update rules of the DoS attacker are as follows:
\begin{equation}\label{dafs1}
\pi_{k+1}^{{n+1}}(s)=\left\{
\begin{aligned}
&\pi_{k}^{n+1}(s),\text{if}\:s\neq s_{k}\:\text{or} \\ &\max_{\bm{\beta}}Q^{n+1}(s,\bm{\beta}|\theta_{k}^{a})=\max_{\bm{\beta}}Q^{n+1}(s,\bm{\beta}|\theta_{k+1}^{a})\\
&\bm{\beta}_{k}, \text{otherwise},
\end{aligned}
\right.
\end{equation}
where the initial strategy $\pi_{1}^{n+1}(s)$ is randomly selected from the {action} set of the DoS attacker.

In the system of this paper, all the consumer electronic devices work together to train a neural network. We denote the Q-value of {state $s \in S$ and action $\alpha$} of consumer electronic devices is $Q(s,\bm{\alpha}|\theta_{k}^{s})$, where $\theta_{k}^{s}$ is the parameter for training the consumer electronic devices' evaluate network. $Q(s,\bm{\alpha}|\theta_{k}^{s-})$ is the Q-value of the consumer electronic devices' target Q-network, where $\theta_{k}^{s-}$ is the target-network's parameter. The update of Q-network can be expressed as
\begin{equation}\label{ds12f}
\tilde{y}^{s}=\left\{
\begin{aligned}
&Q(s,\bm{\alpha},\theta_{k}^{s-}) ,\text{if}~ (s,\bm{\alpha})\neq(s_{k},\bm{\alpha}_{k}),\\
&r_{k}+\rho \min_{\bm{\alpha}}Q(s_{k+1},\bm{\alpha},\theta_{k}^{s-}) ,  \text{otherwise}.
\end{aligned}
\right.
\end{equation}
\begin{equation}\label{dspio}
\mathcal{L}_{k}^{s}(\theta_{k}^{s})=E_{s,\bm{\alpha},r,s'}[\tilde{y}^{s}-Q(s,\bm{\alpha}|\theta_{k}^{s})^{{2}}],
\end{equation}
 where $s_{k+1}$ is determined based on the current state $s_{k}$ and the action $\alpha$. $\tilde{y}^{s}$ is the target or TD-target of the consumer electronic devices, $\mathcal{L}_{k}^{s}(\theta_{k}^{s})$ is the loss function of the Q-network.
Similarly, we utilize SGD to optimize the loss function and obtain the weight parameters of the DQN, expressed as
\begin{equation}\label{dfa45}
\theta_{k+1}^{s}=\theta_{k}^{s}-\eta^{s}\nabla_{\theta_{k}^{s}}\mathcal{L}_{k}(\theta_{k}^{s}),
\end{equation}
where $\eta^{s}$ {denotes} the learning rate and $\nabla_{\theta_{k}^{s}}\mathcal{L}_{k}(\theta_{k}^{s})$ is the weight gradient of the loss function, i.e.,
\begin{equation}
\nabla_{\theta_{k}^{s}}\mathcal{L}_{k}(\theta_{k}^{s})=E_{s,\bm{\alpha},r,s'}[\tilde{y}^{s}-Q(s,\bm{\alpha}|\theta_{k}^{s})].
\end{equation}

The policy update rules of the consumer electronic devices are proposed as
\begin{equation}\label{dfa123}
\pi_{k+1}^{1}(s)=\left\{
\begin{aligned}
&\pi_{k}^{1}(s),\text{if}\:s\neq s_{k}\:\text{or}\\ &\min_{\bm{\alpha}}Q(s,\bm{{\alpha}}|\theta_{k}^{s})=\min_{\bm{\alpha}}Q^{n+1}(s,\bm{\alpha}|\theta_{k+1}^{s})\\
&\bm{\alpha}_{k},\text{otherwise},
\end{aligned}
\right.
\end{equation}
where $\pi_{1}^{1}(s)\in \{0,1\}$ is adopted randomly among all $s \in S$.
\begin{remark}
In the distributed Minimax-DQN, both consumer electronic devices and the DoS attacker train their own network, and they determine their actions based on them. The training process is shown in Algorithm~\ref{alg:algorithm2}. In this way, the model complexity can be reduced. The consumer electronic devices and the DoS attacker each have $2^{n}$ strategies, {thereby greatly reducing the size of the action space from $2^{2n}$ in centralized Minimax-DQN to $2*2^{n}$.} This will effectively increase the convergence rate of the model.
\end{remark}

\begin{algorithm}[tp]
\caption{{Distributed Minimax-DQN for NE}}
\label{alg:algorithm2}
 \begin{algorithmic}[1]
 \State \textbf{Input:} Markov game $\mathcal{G}=<L,S,A,\delta,R,\rho>$, replay memory $\mathcal{D}_{a}$ for attacker, $\mathcal{D}_{s}$ for consumer electronic devices, minibatch size $N$, exploration probability $\epsilon \in (0,1)$.
 \For {episode = $1$ to $T$}
 \State Initial $\mathcal{D}_{a}$,$\mathcal{D}_{s}$, Q-network with parameters $\theta^{a}$,$\theta^{s}$, target-network with parameters $\theta^{a-}=\theta^{a}$,$\theta^{s-}=\theta^{s}$ of the DoS attacker and consumer electronic devices respectively, observation $s_{1}=\bar{\mathbf{P}}$ and $k$=1.
 \While{{$\|Q(s_{k},\bm{\beta}|\theta^{a-}_{k+1})-Q(s_{k},\bm{\beta}|\theta^{a-}_{k})\|>0$ and $\|Q(s_{k},\bm{\alpha}|\theta^{s-}_{k+1})-Q(s_{k},\bm{\alpha}|\theta^{s-}_{k})\|>0$}}
 \State {The} DoS attacker select action $\bm{\beta}_{k}=\mathop{\arg\max}\limits_{\bm{\beta}}{Q^{n+1}(s_{k},\bm{\beta}|\theta_{k}^{a-})}$ with $\epsilon$-greedy and consumer electronic devices select the action $\bm{\alpha}_{k}=\mathop{\arg\min}\limits_{\bm{\alpha}}{Q(s_{k},\bm{\alpha}|\theta_{k}^{s-})}$  with $\epsilon$-greedy.
 \State Using \eqref{qgdda}{,}\eqref{f2a1}{,}\eqref{12aff} to train the DoS attacker's Q-network and update the policy $\pi_{k+1}^{{n+1}}$ of attacker by \eqref{dafs1}.
 \State Using \eqref{ds12f}{,}\eqref{dspio}{,}\eqref{dfa45} to train the consumer electronic devices' Q-network and update the policy $\pi_{k+1}^{1}$ of consumer electronic devices by \eqref{dfa123}.
 \EndWhile
 \State \textbf{EndWhile}
 \State Compute the NE for attacker $\pi_{*}^{2}$ and for consumer electronic devices $\pi_{*}^{1}$.
 \EndFor
 \State \textbf{EndFor}
 \end{algorithmic}
\end{algorithm}

\section{Closed-Loop Security Strategies}\label{close loop}
{In the open-loop case analyzed in the previous section}, we assume that consumer electronic devices and the DoS attacker cannot observe each other's actions, {but} the feedback from remote estimator to the consumer electronic {devices} is accessible to the DoS attacker. However, in the actual wireless network environment, the behavior of both parties can often be obtained through eavesdropping attacks. In this section, we focus on closed-loop situations where consumer electronic devices and the DoS attacker can observe each other's actions, but the DoS attacker does not have access to the feedback sent by the remote estimator to the local consumer electronic device. Therefore, the information on both sides becomes asymmetric. Moreover, both players in the game have the ability to infer or guess the other's current behavior based on each other's historical behavior. Obviously, these dynamic speculations enable attacker to make better choices in subsequent instances.

We denote $a_{i,k}\in\{a_{i}^{0},a_{i}^{1}\}$ as the power of the consumer electronic device $i$ for transmission, where $a_{i}^{0},a_{i}^{1}$ respectively represent the power that consumer electronic device $i$ selects the channel in \textbf{state 0} or \textbf{state 1} to transmit. In {addition}, we let {$b_{i,k}\in \{b_{i}^{0},b_{i}^{1}\}$} to represent the power consumed by a DoS attacker to attack the channel $i$. $b_{i,k}=0$ means that the DoS attacker does not attack the channel $i$ and {$b_{i}^{1}$ is the power needed to attack channel $i$}. We employ signal-interference-plus-noise-ratio (SINR) as a metric to measure the packet loss caused by DoS attacks:
\begin{equation}
\textmd{SINR}_{i,k}=\frac{a_{i,k}}{b_{i,k}+n_{0}},
\end{equation}
where $n_{0}$ is the additive white channel noise's power. Then the packet-error-rate (PER) can be used to measure the packet losses, i.e.,
\begin{equation}\label{dsf125}
\textmd{PER}_{i,k}=\hat{f}(\textmd{SINR}_{i,k}),
\end{equation}
where $\hat{f}(\cdot)$ is a non-increasing function that is determined by the characteristics and modulation schemes being used. The arrival of the packet sent by consumer electronic device $i$ can be {denoted} by $\gamma_{i,k}$. $\gamma_{i,k}=0$ indicates that the data packet is lost, otherwise $\gamma_{i,k}=1$. The probability of successfully receiving a packet is denoted as
\begin{equation}
Pr(\gamma_{i,k}=1)=t_{i,k}\triangleq 1-\textmd{PER}_{i,k}.
\end{equation}

In the asymmetric game, the consumer electronic devices planning problem resembles a MDP, while the DoS attacker planning problem resembles a POMDP. {It is common to form} beliefs as the state of the new MDP to solve the POMDP problems\cite{sondik1978optimal}.

Belief-based {game} can be expressed as a six-tuple $<L,\mathcal{B},\mathcal{A},T,\mathcal{R},{\rho}>$:

\textbf{\textit{Player}}: $L=\{1,\ldots,n,n+1\}$ where $i\in\{1,\ldots,n\}$ represents consumer electronic device $i$ and $n+1$ represent the DoS attacker.

\textbf{\textit{Belief State Space}}: $\tau_{i,k}$ is the redefined as state $s_{k}$ of consumer electronic device $i$ at time $k$. It is the interval between the current time $k$ and the time when the estimator recently successfully received the consumer electronic device packet. $\mathcal{B}=\Delta(S)$ represents the continuous belief state space based on $S$ with $S=\{s_{0},s_{1},\ldots\}$. We consider a finite set to simplify the problem, that is, $\tau_{i,k}\in\{0,1,\ldots,m\}$ for all $i\in\{1,\ldots,n\}$. At time $k$, the DoS attacker generates a predicted probability distribution $\mathbf{B}_{k}\in \mathbb{R}^{n\times (m+1)}$ on state space $S$, where row $i$ of $\mathbf{B}_{k}$ stands for the belief state distribution of consumer electronic device $i$ and $\mathbf{B}_{k}$ is a common knowledge among all players $L$.

\textbf{\textit{Action}}: $\mathcal{A}=\prod_{i\in L} \mathcal{A}_{i}$, where $\mathcal{A}_{i},i \in \{1,\ldots,n\}$ is the action set of consumer electronic device $i$, and $\mathcal{A}_{n+1}$ is the action set of the DoS attacker. At each time step $k$, $\alpha_{k}^{i}\in \{0,1\}$ represents the action of consumer electronic device $i$. $\alpha_{k}^{i}=0$ means that the channel in \textbf{state 0} is selected and $\alpha_{k}^{i}=1$ means that the channel in \textbf{state 1} is selected. Similarly, we denote $\beta_{k}^{i}\in\{0,1\}$ as the variable of the attacker's action on the $i$-th channel.

\textbf{\textit{Transition Probability}}: Let $B_{k}^{i,j}$ be the $ij$-th element in $\mathbf{B}_{k}$, and then the {transition function} is expressed as
\begin{align*} {\textbf{B}}_{k+1} =T_{k}\begin{pmatrix}\frac{t_{1,k}}{1-t_{1,k}} & B_{k}^{1,1} & \cdots & B_{k}^{1,m}+B_{k}^{1,m+1} \\ \vdots & \vdots & \cdots & \vdots \\ \frac{t_{n,k}}{1-t_{n,k}}& B_{k}^{n,1} & \cdots & B_{k}^{n,m}+B_{k}^{n,m+1} \end{pmatrix}, \end{align*}
where {$T_{k}=diag(1-t_{1,k},\ldots,1-t_{n,k})$}. In this finite set, we consider that after $m$ consecutive packet losses, the second packet loss is still counted as $m$ packet losses. It is easy to obtain that the belief state $\mathbf{B}$ satisfies Markov processes.

\textbf{\textit{Reward function}}: $\mathcal{R}=\{r_{1},\ldots,r_{k},\ldots\}$. The instantaneous reward function $r_{k}$ can be computed as
\begin{align}
r_{k}&=r(\mathbf{B}_{k},{\bm{\alpha_{k}},\bm{\beta_{k}}})\\\nonumber
&=\sum \limits _{i=1}^{n}\sum _{j=1}^{m+1}B_{k}^{i,j}(j-1)+c_{i}\alpha_{k}^{i}-c^{i}_{\beta}\beta^{i}_{k}.
\end{align}

{For consumer electronic devices, the expectation is minimum belief state at minimum cost, and DoS attacker expects minimum information state at minimum cost.}

{\textbf{\textit{Discount factor}}: $\rho \in (0,1)$.}
\begin{remark}
The {centralized and} distributed Minmax-DQN proposed in the open-loop case also can be apply to the closed loop case. It just need to change the state $s$ to the belief state $\mathbf{B}$ and modify the reward function $r(\cdot)$.  In the system of this paper, a target is measured by multiple consumer electronic devices and then evaluated jointly in the remote estimator. Therefore, the belief state $\mathbf{B}$ of the DoS attacker is a matrix, which is different from a vector in \cite{ding2018attacks}.
\end{remark}

\begin{table*}
\centering
\caption{Convergence of {$Q(s_{k},\alpha|\theta_{k}^{s})$ and $Q(s_{k},\beta|\theta_{k}^{a})$} in two states}
\label{tab:table2}
\resizebox{1.0\linewidth}{!}{
\begin{tabular}{ccccccccc}
\toprule
\multirow{2}{*}{State}                                             & \multicolumn{4}{c}{{$Q(s_{k},\alpha|\theta_{k}^{s})$}}                             & \multicolumn{4}{c}{{$Q(s_{k},\beta|\theta_{k}^{a})$}}                              \\
\cline{2-9}
                                                                   & $\bm{\alpha}=(0,0)$ & $\bm{\alpha}=(0,1)$ & $\bm{\alpha}=(1,0)$ & $\bm{\alpha}=(1,1)$ &  $\bm{\beta}=(0,0)$ &   $\bm{\beta}=(0,1)$ &   $\bm{\beta}=(1,0)$ &   $\bm{\beta}=(1,1)$   \\
\hline
$\bar{\mathbf{P}}$                                                 &\textbf{ 1.356 }     & 5.549      & 7.744      & 12.264     & \textbf{6.404}      & 0.763      & 1.566      & -3.878      \\
$F(\bar{\mathbf{P}}, \begin{bmatrix}1 & 0 \\ 0 & 0 \end{bmatrix})$ &\textbf{ 1.047  }    & 5.774      & 7.834      & 12.446     &\textbf{ 5.920  }    & 0.616      & 1.232      & -3.529      \\
\bottomrule
\end{tabular}
}
\end{table*}

\section{Simulation and Analysis}\label{simulation}
\vspace{1.0em}
This section first describes the experimental parameter settings of the system. Subsequently, we verify the proof presented earlier through experiments. Finally, we will do some experiments to prove the feasibility and superiority of our algorithm in open-loop case and close-loop case.
\subsection{System Parameter Setting}
We consider the system with two consumer electronic devices and a DoS attacker with parameters:\\
\begin{align*}
&A=\begin{bmatrix}2 & 1 \\ 0.7 & 0.8 \end{bmatrix},C=\begin{bmatrix}1 &0 \\ 0 & 2 \end{bmatrix},Q=\begin{bmatrix}0.6 &0 \\ 0 & 0.6\\\end{bmatrix},\\
&R_{1}=0.7,R_{2}=0.4,
\end{align*}
the initial system status is the steady-state error covariance $\bar{\mathbf{P}}=\begin{bmatrix}0.530 &0.020 \\ 0.020 & 0.088 \end{bmatrix}$.\\

{In order to verify Lemma 2 and demonstrate the error covariance's rapid convergence to $\bar{\mathbf{P}}$ without packet loss, we design an experiment comparing the remote estimation system's performance before and after DoS attacks. As shown in Fig. \ref{model_correct_proof}, the stability of the system is measured by the estimation error covariance which is always $\bar{\mathbf{P}}$ in the steady state without packet loss.} For the system in an unstable state, that is, packets are lost for five consecutive communication cycles starting from step=1, the {estimation error covariance} will start to increase. But when packet loss is stopped from step=6, the {estimation error covariance} can quickly converge to $\bar{\mathbf{P}}$, so that the system is in a stable state.
\begin{figure}
\flushleft
\includegraphics[width=1\linewidth]{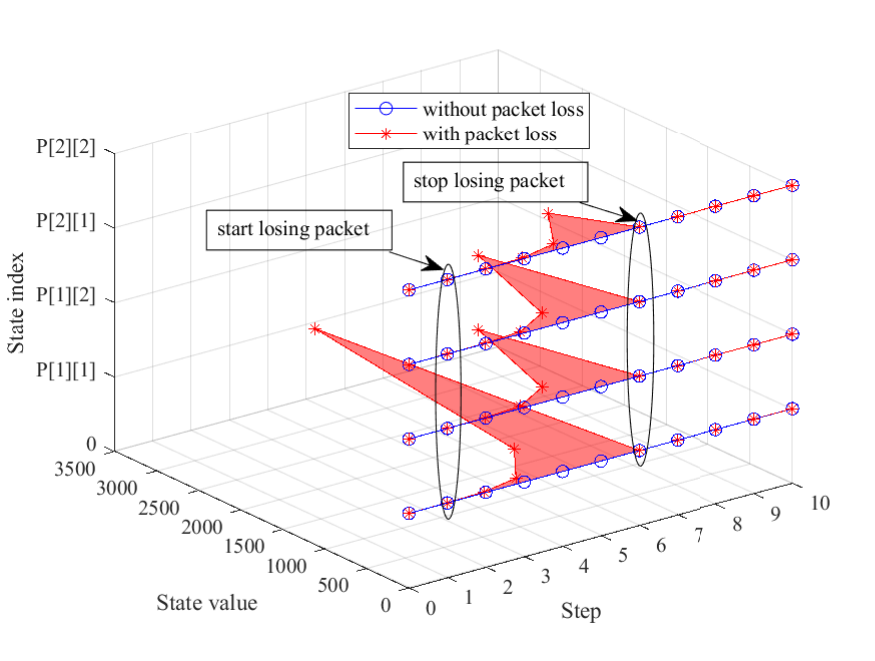}
\caption{The correctness verification of the designed model.} \label{model_correct_proof}
\end{figure}

\subsection{Open-loop Case}
\begin{figure*}[tp]
\centering
\includegraphics[width=1\linewidth]{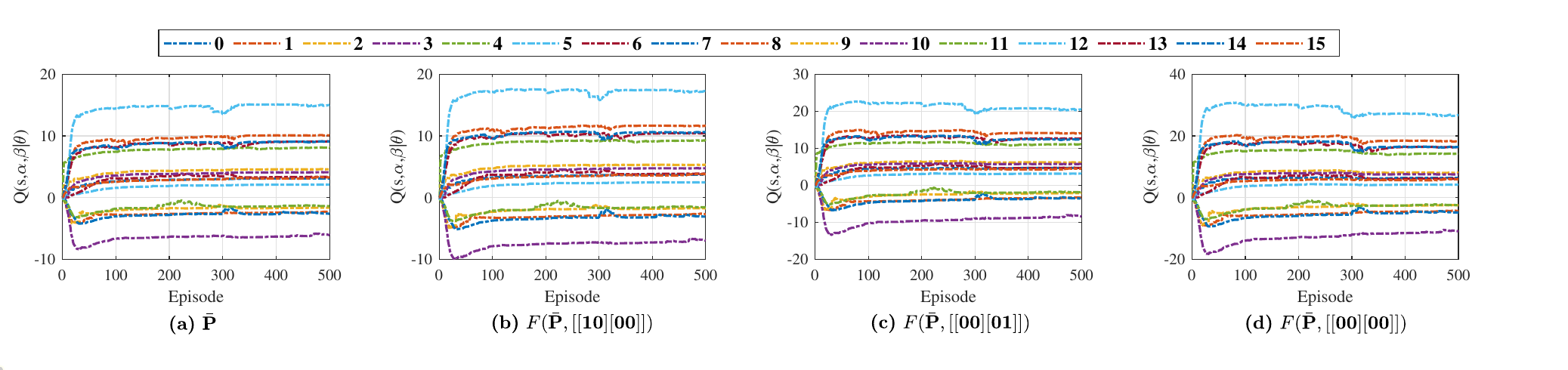}
\caption{{$Q(s,\bm{\alpha},\bm{\beta}|\theta)$ in four state duiring centralized learning under open-loop case, respectively.}}\label{open_loop_contribute}
\end{figure*}

By transmitting data through the channel in \textbf{state 1}, the consumer electronic device $1$ consumes $c_{1}=7$ and the consumer electronic device $2$ consumes $c_{2}=5$. They have no consumption through the channel in \textbf{state 0}. The DoS attacker attacks the first and the second channel represently cost $c_{\beta}^{1}=c_{\beta}^{2}=6$. We set the learning rate of the network to be 0.1 with a discount factor of 0.8 and $\epsilon =0.9$. The combination of consumer electronic devices and attacker actions is set as \{0,1,...,15\}, whose the first two bits represent $\bm{\alpha}$ and the last two represent $\bm{\beta}$.

Fig.\ref{open_loop_contribute} shows the $Q(s,\bm{\alpha},\bm{\beta}|\theta)$ at state $\bar{\mathbf{P}}$, $F(\bar{\mathbf{P}},\begin{bmatrix}1 & 0 \\ 0 & 0 \end{bmatrix}),F(\bar{\mathbf{P}},\begin{bmatrix}0 & 0 \\ 0 & 1 \end{bmatrix}),F(\bar{\mathbf{P}},\begin{bmatrix}0 & 0 \\ 0 & 0 \end{bmatrix})$ in the learning process by centralized Minimax-DQN. Each color line represents the Q value of each action combination. We can see that of all states in Fig. \ref{open_loop_contribute}, the Q value of action 3 is the smallest after training. Therefore, consumer electronic devices are more interested in selecting the action corresponding to this Q value, i.e. $\bm{\alpha}=(0,0)$. Considering the action of consumer electronic devices, the Q value of action 0 is the largest, so the attacker is more inclined to choose the action corresponding to the Q value, that is, $\bm{\beta}=(0,0)$. The NE in these states is that consumer electronic devices selects \textbf{state 0} of channel $1$ and channel $2$ to transmit, while the attacker doesn't attack channel $1$ and channel $2$.

In distributed Minimax-DQN, consumer electronic devices and the DoS attacker use their own neural networks to learn. We set the learning rate of the DoS attacker network to be 0.01 with a discount factor of 0.8, and the learning rate of the consumer electronic device network to be 0.01 with a discount factor of 0.8. Q values of the consumer electronic devices and attacker in $\bar{\mathbf{P}}$ and $F(\bar{\mathbf{P}}, \begin{bmatrix}1 & 0 \\ 0 & 0 \end{bmatrix})$ states are shown in Fig. \ref{open_loop_distribute}. As can be seen from the figure, consumer electronic devices select the \textbf{state 0} of channel $1$ and channel $2$ for transmission, and the attacker chooses not to attack the channel is the NE of the two states. The same NE is found using distributed Minimax-DQN as using centralized Minimax-DQN. Also we can see that both states in Fig. \ref{open_loop_distribute} , the Q values of consumer electronic devices and attacker eventually converge and the convergency values are shown in Table \ref{tab:table2}, {where the bold is the convergence of Q-value reaching NE}. From  Fig. \ref{open_loop_distribute_loss}, we can see that the neural network loss of consumer electronic devices and the attacker has reached convergence before 500 iterations and the convergence values of \textbf{1.066} and \textbf{1.191}, respectively. This shows that distributed Minimax-DQN performs well in convergence speed. The  Fig. \ref{compare_strategy} shows that the centralized and distributed Minimax-DQN proposed can find NE almost simultaneously from the steady state, which is faster than the NE Q-learning algorithm proposed by Ding et al\cite{ding2020defensive}. {However, this paper uses the data of the experience replay pool, enabling to enhance exploration efficiency in the high-dimensional state space, thus enabling to find a more reasonable and superior NE.} It can be seen from the figure that the overall NE Q-learning algorithm still has large fluctuations in strategy selection, while the two algorithms proposed are relatively stable. In terms of the speed and stability of finding NE, the two algorithms proposed are equivalent to Distributed Reinforcement Learning Algorithm of Dai et al.\cite{dai2020distributed}. {Under open-loop experimental configuration, consumer electronic devices take into account the rapid convergence of the Kalman filter after an error, which enables consumer electronic devices to effectively maintain the stability and performance of the system without frequently selecting \textbf{state 1}.  Therefore, consumer electronic devices are more willing to choose the \textbf{state 0} for transmission. The DoS attacker will also consider the problem of rapid system recovery. In order to avoid wasting energy and resources in continuous attacks, the DoS attacker is more willing to choose the strategy of no attack by choosing a right time attack to improve the attack efficiency.}

\begin{figure*}[tp]
\centering
\includegraphics[width=1\linewidth]{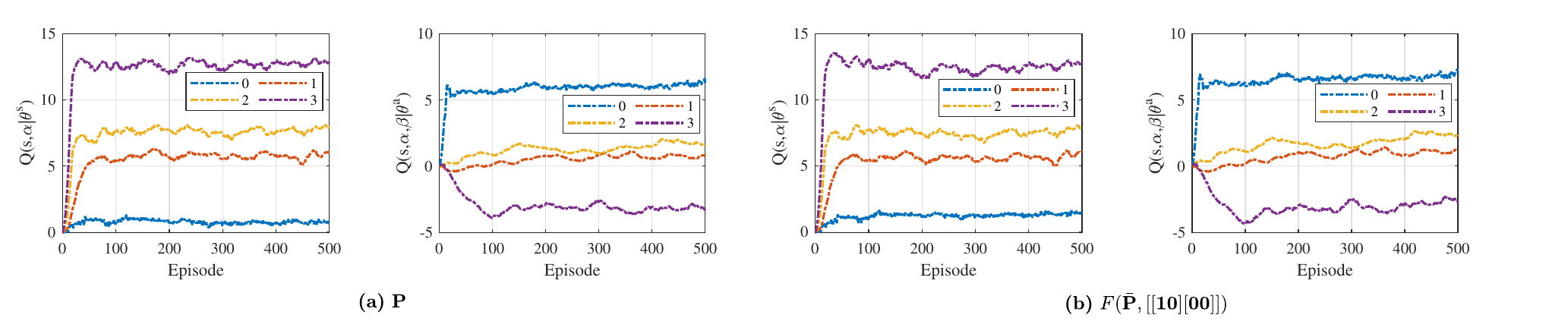}
\caption{{Q-value of consumer electronic devices and DoS attacker in two states during the distributed learning under open-loop case, respectively.}} \label{open_loop_distribute}
\end{figure*}

\begin{figure}[tp]
\flushleft
\includegraphics[width=1\linewidth]{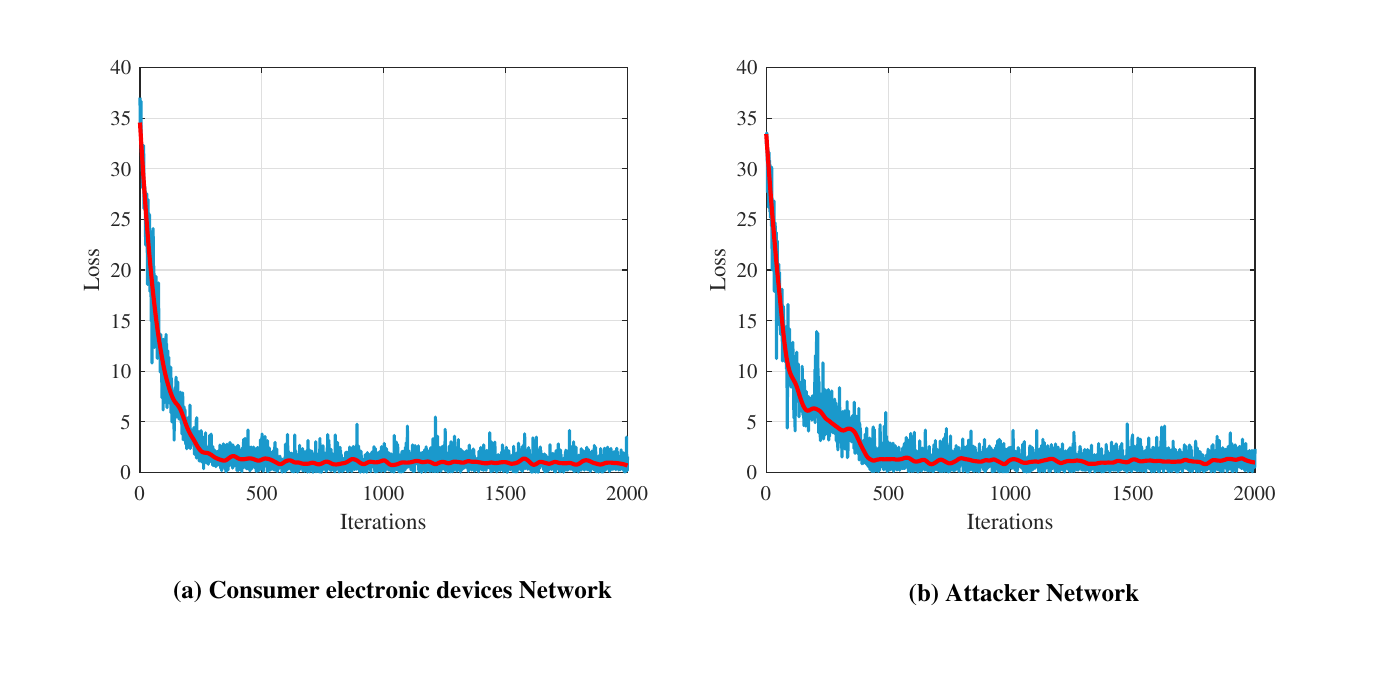}
\caption{The neural network loss of consumer electronic devices and the attacker.} \label{open_loop_distribute_loss}
\end{figure}
\begin{figure}[tp]
\flushleft
\includegraphics[width=1\linewidth]{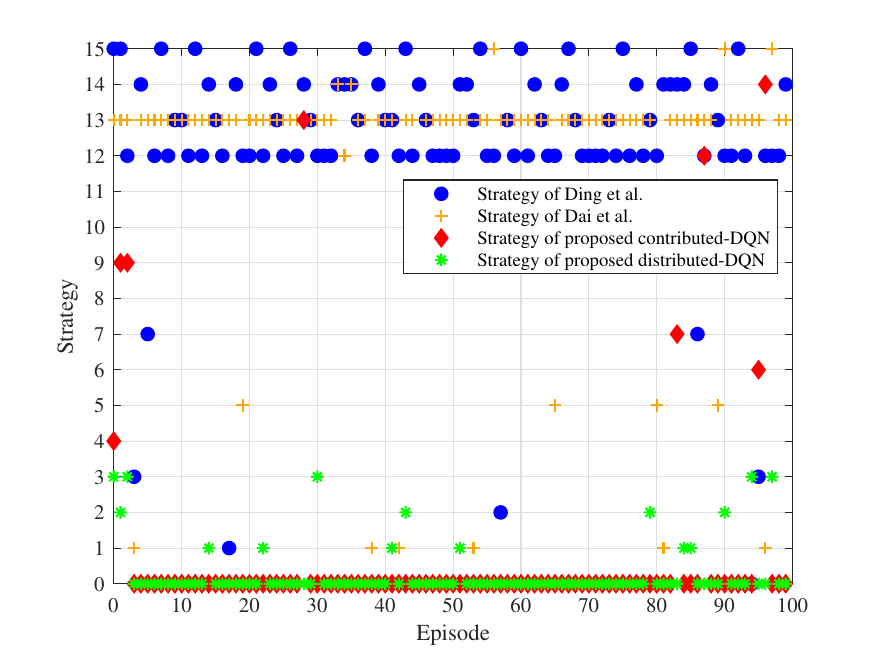}
\caption{Comparison of the performance of different methods to find NE.} \label{compare_strategy}
\end{figure}
\begin{figure*}[tp]
\centering
\includegraphics[width=1\linewidth]{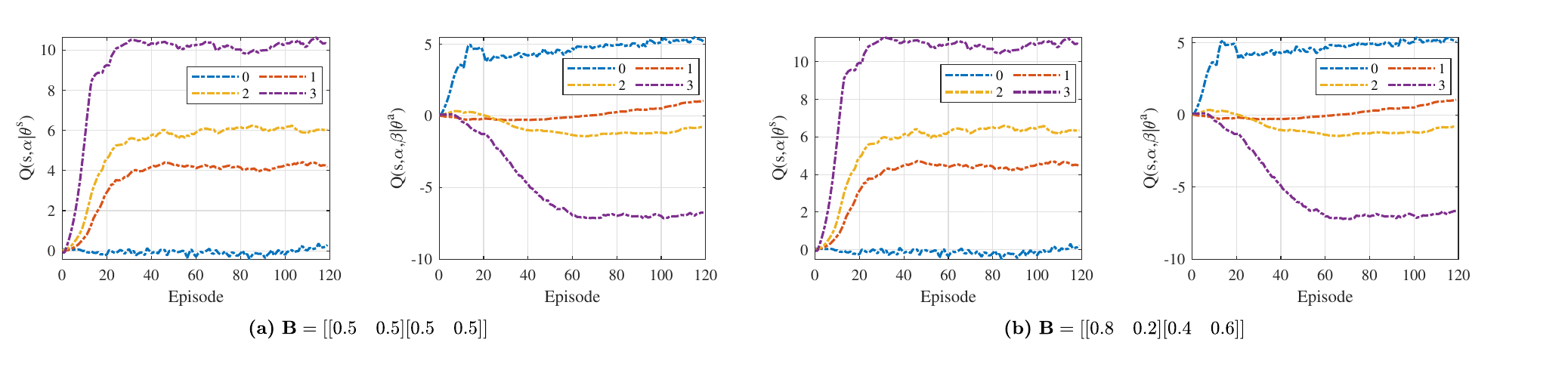}
\caption{{Q-value of consumer electronic devices and DoS attacker in two states during the distributed learning under closed-loop case, respectively.}} \label{close_loop_distribute}
\end{figure*}

\subsection{Closed-loop Case}
In this section, {the more advantageous} distributed Minimax-DQN is used for experiments based on belief state space in a closed loop, considering two consumer electronic devices and a DoS attacker. To reduce computation, we define the set of states as a finite set $\{0,1... m\}$ and take $m=1$. It is defined that the power required for consumer electronic device $1$ to transmit through channel $1$ in \textbf{state 0} is $a_{1}^{0}=0.3$ and in \textbf{state 1} is $a_{1}^{1}=0.7$. The power required for the consumer electronic device to select a channel in \textbf{state 0} is $a_{2}^{0}=0.2$, and in \textbf{state 1} is $a_{2}^{1}=0.8$. Similarly, an attacker is defined to consume 0.5 power to attack any channel, and no power is required if no attack is performed. That is {$b_{1}^{0}=b_{2}^{0}=0$ and $b_{1}^{1}=b_{2}^{1}=0.5$}. At the same time, assuming the channel noise $n_{0}\sim\mathcal{N}(0,0.1)$, the function of \eqref{dsf125} is defined as $f(x)=\frac{1}{e^{x}}$. The initial belief state is $\mathbf{B}_{0}=\begin{bmatrix}0.5 &0.5 \\ 0.5 & 0.5 \end{bmatrix}$. In the closed-loop distributed Minimax-DQN, other network parameters are the same as in the open-loop distributed Minimax-DQN.

In this paper, we make an experiment that the Q value change of DoS attacker and consumer electronic devices in a state of belief $\mathbf{B}=\begin{bmatrix}0.5 &0.5 \\ 0.5 & 0.5 \end{bmatrix}$ and $\mathbf{B}=\begin{bmatrix}0.8 &0.2 \\ 0.6 & 0.4 \end{bmatrix}$, as shown in Fig. \ref{close_loop_distribute}. {A large number of experimental data reveals the following results: In belief state $\mathbf{B}=[[0.5,0.5],[0.5,0.5]]$, the attacker's largest value of $Q(s,\bm{\beta}|\theta_{k}^{a})$ is $\mathbf{5.305}$, and the consumer electronic device's smallest value of $Q(s,\bm{\alpha}|\theta_{k}^{s})$ is $\mathbf{0.197}$. In the belief state $\mathbf{B}=[[0.8,0.2],[0.6,0.4]]$, the attacker's largest value of $Q(s,\bm{\beta}|\theta_{k}^{a})$ is $\mathbf{5.216}$, and the consumer electronic device's smallest value of $Q(s,\bm{\alpha}|\theta_{k}^{s})$ is $\mathbf{0.131}$}. The algorithm basically converges in the final and the NE of both states $\mathbf{B}=[[0.5,0.5],[0.5,0.5]]$ and $\mathbf{B}=[[0.8,0.2],[0.6,0.4]]$ is that consumer electronic devices adopt the channel in \textbf{state 0} for transmission and the attacker does not attack any channel.
\section{Conclusions}\label{conclusion}
\vspace{1.0em}
In this paper, we proposed a distributed RSE model tailored for electronic consumer IoT, addressing critical security challenges posed by DoS attacks. By leveraging multiple consumer electronic devices to measure the same system target, the model utilized a centralized Kalman filter at the remote estimator, effectively reducing consumer electronic device computational load and mitigating risks of data leakage. To address the adversarial strategies between consumer electronic devices and DoS attackers, we introduced centralized and distributed Minimax-DQN algorithms, employing NE frameworks under both open-loop and closed-loop scenarios. These methods demonstrated superior adaptability to high-dimensional data and complex environments compared to traditional Q-learning solutions. Experimental results validated the effectiveness and stability of our approach, showing faster convergence and improved performance in finding NE. {The ability of centralized and distributed Minimax-DQN to schedule policies from both sides of the offense and defense in resource-constrained environments further enhances its practicality for large-scale deployment.} This work provides a robust foundation for enhancing the security and scalability of IoT networks, contributing to the development of secure, real-time monitoring and decision-making systems in consumer electronics. Future research can extend the model applications in predictive maintenance, scalability, and advanced IoT scenarios.

\section{Acknowledgments}
This work was supported by the Science and Technology Major Project of Tibetan Autonomous Region of China under Grant No. XZ202201ZD0006G02, the Basic and Applied Basic Research Foundation of Guangdong Province (No.2024A1515012398) and the National Innovation and Entrepreneurship Training Program For Undergraduate No. 202310559107. The work of Zahid Khan and Wadii Boulila is supported by Prince Sultan University, Riyadh, Saudi Arabia.

\vspace{0.9em}

\vspace{0.9em}

\end{CJK}
\end{document}